\documentclass{sig-alternate-05-2015}

\CopyrightYear{2016} \setcopyright{acmcopyright}
\conferenceinfo{WSDM'16,}{February 22--25, 2016, San Francisco, CA, USA.}
\isbn{978-1-4503-3716-8/16/02}\acmPrice{\$15.00}
\doi{http://dx.doi.org/10.1145/2835776.2835777} 

\usepackage{etex}
\usepackage{subfigure}
\usepackage{graphicx}
\usepackage{amsmath}
\usepackage{bm}
\usepackage{url}
\usepackage{stmaryrd}
\usepackage{balance}
\usepackage{amssymb}
\usepackage{pgfplots}
\usepackage{pifont}
\usepackage[usenames,dvipsnames]{pstricks}
\usepackage{epsfig}
\usepackage{epstopdf}
\usepackage{multirow}
\usepackage{booktabs}
\usepackage{pgffor}
\usepackage{tikz}
\usepackage{algorithm,algorithmic}

\newtheorem{definition}{Definition}[section]
\newtheorem{theorem}{Theorem}[section]

\newtheorem{lemma}[theorem]{Lemma}

\newcommand{\ml}[0]{\textsc{MovieLens}}

\newcommand{\bing}[0]{\textsc{BingAds}}

\newcommand{\mvm}[0]{\textsc{MVM}}
\newcommand{\fm}[0]{\textsc{FM}}
\newcommand{\bsfm}[0]{\textsc{mvFM}}
\newcommand{\threewayfm}[0]{\textsc{mvFM-3d}}
\newcommand{\partialmvm}[0]{\textsc{mvFM-reg}}
\newcommand{\tf}[0]{\textsc{TF}}
\newcommand{\lr}[0]{\textsc{LR}}

\long\def\comment#1{}

\begin{document}
\title{Multi-View Factorization Machines}
\numberofauthors{1}
\author{
Bokai~Cao$^1$\thanks{This work was done while the author was doing internship at Microsoft Research.}~,
Hucheng~Zhou$^2$,
Guoqiang~Li$^3$\thanks{This work was done before the author joins Huawei.}~~and
Philip~S.~Yu$^{1,4}$\\
\affaddr{$^1$Department of Computer Science, University of Illinois at Chicago, IL, USA}\\
\affaddr{$^2$Microsoft Research, Beijing, China}\\
\affaddr{$^3$Huawei Technologies, Shenzhen, China}\\
\affaddr{$^4$Institute for Data Science, Tsinghua University, Beijing, China}\\
\email{caobokai@uic.edu, huzho@microsoft.com, liguoqiang9@huawei.com, psyu@cs.uic.edu}
}

\maketitle

\begin{abstract}
With rapidly growing amount of data available on the web, it becomes increasingly likely to obtain data from different perspectives for multi-view learning. Some successive examples of web applications include recommendation and target advertising. Specifically, to predict whether a user will click an ad in a query context, there are available features extracted from user profile, ad information and query description, and each of them can only capture part of the task signals from a particular aspect/view. Different views provide complementary information to learn a practical model for these applications. Therefore, an effective integration of the multi-view information is critical to facilitate the learning performance.

In this paper, we propose a general predictor, named multi-view machines (MVMs), that can effectively explore the full-order interactions between features from multiple views. A joint factorization is applied for the interaction parameters which makes parameter estimation more accurate under sparsity and renders the model with the capacity to avoid overfitting. Moreover, MVMs can work in conjunction with different loss functions for a variety of machine learning tasks. The advantages of MVMs are illustrated through comparison with other methods for multi-view prediction, including support vector machines (SVMs), support tensor machines (STMs) and factorization machines (FMs).

A stochastic gradient descent method and a distributed implementation on Spark are presented to learn the MVM model. Through empirical studies on two real-world web application datasets, we demonstrate the effectiveness of MVMs on modeling feature interactions in multi-view data. A 3.51\% accuracy improvement is shown on MVMs over FMs for the problem of movie rating prediction, and 0.57\% for ad click prediction.
\end{abstract}


\section{Introduction}\label{sec:intro}

Web data is available not only in great volume but also in multiple representations/views from a variety of sources or feature subsets. Generally, different views provide complementary information to learn an effective model for web-scale applications. Thus, multi-view learning can facilitate the learning performance and is prevalent in a wide range of application domains. For example, for the business on the web, it is critical to estimate the probability that the display of an ad to a specific user when s/he searches for a query will lead to a click. This process involves three entities: users, ads, and queries. An effective integration of features describing these different entities is directly related to precise targeting of an advertising system.

One of the key challenges of multi-view learning is to model the interactions/correlations between different views, wherein complementary information is contained. Conventionally, multi-kernel learning algorithms combine kernels associated with respective views to improve the learning performance \cite{lanckriet2004learning}. Basically, coefficients are learned based on the usefulness/informativeness of the corresponding views, and thus inter-view correlations are only considered at the view-level. These approaches, however, fail to explore the explicit correlations between features across multiple views.

In contrast to modeling on views, another direction for modeling multi-view data is to directly consider the abundant correlations between features from different views. Feature interactions with different orders can reflect different but complementary insights. Assume that we have obtained a latent factor representing wealth/price-related attributes for each entity ({\em i.e.}, users, ads and queries) in the advertising system of a search engine, as illustrated in Table~\ref{tab:example}. For example, users who have high purchase power ({\em i.e.}, a positive latent factor, $a_{user}>0$) may have interests in luxury products ($a_{ad}>0$). However, a thoughtful recommender system should not always recommend luxury products to these users regardless of the query context. In Table~\ref{tab:example}, it is unreasonable to recommend a luxury bag ($a_{ad}>0$) to the user when s/he searches for a disease ($a_{query}<0$), in which case some relevant medicines or medical books ($a_{ad}<0$) would seem better choices. We can observe that, in such scenarios, only the third-order interactions contribute positively to the recommendation of medicines and negatively to that of luxury bags, while the first-order and the second-order interactions insist to recommend something inappropriate in the specific context. Note here that we do not claim the higher order interactions can work as the best indicator by their own in all problems. Nevertheless, integrating their contributions into the decision function in an efficient manner is critical.

\begin{table}
\centering
\caption{An example showing the discrepancy between feature interactions with different orders. $\#1=user+ad+query$, $\#2=user\times ad+user\times query+ad\times query$, $\#3=user\times ad\times query$.}
\label{tab:example}
\begin{tabular}{rrr|rrr}
\toprule
User & Ad & Query & \#1 & \#2 & \#3 \\
\midrule
1.20 \color{green}{\tiny(+)} & 1.80 \color{green}{\tiny(+)} & 0.50 \color{green}{\tiny(+)}
& 3.50 \color{green}{\tiny(+)} & 3.66 \color{green}{\tiny(+)} & 1.08 \color{green}{\tiny(+)} \\
1.20 \color{green}{\tiny(+)} & 1.80 \color{green}{\tiny(+)} & -0.50 \color{red}{\tiny(-)}
& 2.50 \color{green}{\tiny(+)} & 0.66 \color{green}{\tiny(+)} & -1.08 \color{red}{\tiny(-)} \\
1.20 \color{green}{\tiny(+)} & -1.80 \color{red}{\tiny(-)} & -0.50 \color{red}{\tiny(-)}
& -1.10 \color{red}{\tiny(-)} & -1.86 \color{red}{\tiny(-)} & 1.08 \color{green}{\tiny(+)} \\
\bottomrule
\end{tabular}
\end{table}

S. Rendle pioneers the concept of factorization machines (FMs) \cite{rendle2010factorization} which are now the state-of-the-art approach to model feature interactions and inspire this work. However, the practical implementations of FMs are usually limited to the second-order interactions, {\em i.e.}, pairwise correlations. This is partially due to the fact that a separate set of latent factors (parameters to be learned) is introduced for each order of interactions in FMs. That is to say, a feature has a latent representation when it is considered for the second-order interactions, while the same feature has a different representation for the third-order interactions. Moreover, the global bias and the first-order interaction terms in FMs are not factorized and independent from the latent factors for higher order interactions. These bias terms and latent factors for different orders altogether compose inconsistent representations of input features and thus compromise the model interpretability. In addition, independent factorization of interactions with different orders results in a large set of model parameters to be learned which makes the training process challenging.

The major challenge of including higher order interactions is that observations with such interactions become sparser with higher orders. Therefore, parameters representing higher order interactions can hardly be learned from their limited observations, especially from the extremely sparse data, {\em e.g.}, recommender systems. We suggest a common latent subspace for all features that is shared by different orders of interactions. In this manner, the full-order interactions observed in the data can collectively be used to learn a consistent representation in the latent feature space.

In this paper, we propose a novel model for multi-view prediction, called multi-view machines (MVMs). The main advantages of MVMs are outlined as follows:
\begin{itemize}
\item MVMs include the global bias and the full-order interactions between features from multiple views, ranging from the first-order interactions ({\em i.e.}, contributions of single features) to the highest-order interactions ({\em i.e.}, contributions of combinations of features from each view).
\item MVMs jointly factorize the interaction parameters for different orders to allow accurate parameter estimation under sparsity and avoid overfitting via the effect of bias factors.
\item MVMs are a general predictor that can work with different loss functions ({\em e.g.}, square error, hinge loss, logit loss) for a variety of machine learning tasks.
\end{itemize}

To empirically analyze and understand these advantages, we have the MVM model and other baselines implemented in a distributed environment, GraphX \cite{graphx}, which is a component of Spark \cite{spark}. Extensive experiments are conducted on real-world web application datasets, for regression and classification tasks, respectively. A 3.51\% accuracy improvement is shown on MVMs over FMs for the problem of movie rating prediction, and 0.57\% for ad click prediction.

\section{Background}\label{sec:prelim}

In this section, we first state the problem of multi-view prediction and briefly review the adaptation of existing methods for multi-view prediction, including support vector machines (SVMs), support tensor machines (STMs) and factorization machines (FMs).

\subsection{Multi-view Prediction}\label{sec:problem}

Suppose each instance has representations in $m$ different views, {\em i.e.}, $\mathbf{x}^T = \left( {\mathbf{x}^{(1)}}^T, ..., {\mathbf{x}^{(m)}}^T \right)$, where $\mathbf{x}^{(p)} \in \mathbb{R}^{I_p}$, $I_p$ is the dimensionality of the $p$-th view. Let $d=\sum_{p=1}^{m} I_p$, so $\mathbf{x} \in \mathbb{R}^d$. Considering the problem of click through rate (CTR) prediction for advertising display, for example, an instance corresponds to an {\em impression} which involves a user, an ad, and a query. Therefore, suppose $\mathbf{x}^T = \left( {\mathbf{x}^{(1)}}^T, {\mathbf{x}^{(2)}}^T, {\mathbf{x}^{(3)}}^T \right)$ is an impression, $\mathbf{x}^{(1)}$ contains information of the user profile, $\mathbf{x}^{(2)}$ is associated with the ad information, and $\mathbf{x}^{(3)}$ is the description from the query aspect. The result of an impression is {\em click} or {\em non-click}. 

Given a training set with $n$ labeled instances represented from $m$ views: $\mathcal{D} = \left\{ \left( \mathbf{x}_i, y_i \right)|~i = 1,...,n \right\}$, in which $\mathbf{x}_i^T = \left( {\mathbf{x}_i^{(1)}}^T, ..., {\mathbf{x}_i^{(m)}}^T \right)$ and $y_{i}\in\{-1,1\}$ is the class label of the $i$-th instance. For CTR prediction problem, $y=1$ denotes {\em click} and $y=-1$ denotes {\em non-click} in an impression. The task of multi-view classification is to learn a function $f: \mathbb{R}^{I_1} \times \cdots \times \mathbb{R}^{I_m} \rightarrow \{-1,1\}$ that correctly predicts the label of a test instance. Alternatively, if $y_{i}\in\mathbb{R}$, it is a multi-view regression problem, {\em e.g.}, rating prediction.

\begin{table}
\small
\caption{Symbols.}
\label{tab:notation}
\begin{tabular}{ll}
\toprule
Symbol & Definition and description\\
\midrule
$s$ & each lowercase letter represents a scale\\
$\mathbf{v}$ & each boldface lowercase letter represents a vector\\
$\mathbf{M}$ & each boldface capital letter represents a matrix\\
$\mathcal{T}$ & each calligraphic letter represents a tensor, set or space\\
$\left\langle\cdot,\cdot\right\rangle$ & denotes inner product\\
$\circ$ & denotes tensor product or outer product\\
$\times_k$ & denotes mode-$k$ product\\
$\left|\cdot\right|$ & denotes absolute value\\
$\left\|\cdot\right\|_{F}$ & denotes (Frobenius) norm of vector, matrix or tensor\\
\bottomrule
\end{tabular}
\end{table}

Table~\ref{tab:notation} lists some basic symbols that will be used throughout the paper. In addition, we introduce the concept of tensors which are higher order arrays that generalize the notions of vectors (the first-order tensors) and matrices (the second-order tensors), whose elements are indexed by more than two indexes. Definitions of tensor product and mode-$k$ product are given which will be used to formulate our proposed model.

\begin{definition}[Tensor Product or Outer Product]
The tensor product $\mathcal{X}\circ\mathcal{Y}$ of a tensor $\mathcal{X} \in \mathbb{R}^{I_1\times\cdots\times I_m}$ and another tensor $\mathcal{Y} \in \mathbb{R}^{I'_1\times\cdots\times I'_{m'}}$ is defined by
\begin{align}
(\mathcal{X}\circ\mathcal{Y})_{i_1,...,i_m,i'_1,...,i'_{m'}}=x_{i_1,...,i_m}y_{i'_1,...,i'_{m'}}
\end{align}
for all index values.
\end{definition}

\begin{definition}[Mode-$k$ Product]
The mode-$k$ product $\mathcal{X}\times_k\mathbf{M}$ of a tensor $\mathcal{X} \in \mathbb{R}^{I_1\times\cdots\times I_m}$ and a matrix $\mathbf{M} \in \mathbb{R}^{I'_k\times I_k}$ is defined by
\begin{align}
(\mathcal{X}\times_k\mathbf{M})_{i_1,...,i_{k-1},j,i_{k+1},...,i_m}=\sum_{i_k=1}^{I_k}x_{i_1,...,i_m}m_{j,i_k}
\end{align}
for all index values.
\end{definition}

Basically, through the mode-k matrix product, a new tensor of the same order is obtained by applying the matrix to each mode-k fiber of the tensor.

\subsection{SVM Model}\label{sec:svm}

Vapnik introduced support vector machines (SVMs) \cite{vapnik2000nature} based on the maximum-margin principle. Essentially, SVMs integrate the hinge loss and the L2-norm regularization. The decision function with a linear kernel is as follows\footnote{The sign function is omitted, because the analysis and conclusions can easily extend to other generalized linear models, {\em e.g.}, logistic regression.}:
\begin{align}
\hat{y}=w_0+\sum_{i=1}^d w_i x_i=w_0+\sum_{p=1}^m\sum_{i_p=1}^{I_p} w_{i_p}^{(p)} x_{i_p}^{(p)}\label{eq:svm}
\end{align}
where $\mathbf{x}$ is simply a concatenation of features from different views in the multi-view setting, {\em i.e.}, $\mathbf{x}^T = \left( {\mathbf{x}^{(1)}}^T, ..., {\mathbf{x}^{(m)}}^T \right)$, as shown in Figure~\ref{fig:related}.


Obviously, no interactions between views are explored in Eq.~(\ref{eq:svm}). Through the employment of a nonlinear kernel, SVMs can implicitly project data from the feature space into a more complex high-dimensional space, which allows SVMs to model higher order interactions between features. However, as discussed in \cite{rendle2010factorization}, all interaction parameters of nonlinear SVMs are completely independent.

For nonlinear SVMs, there must be enough instances $\mathbf{x}\in\mathcal{D}$ where $x_{i_p}^{(p)}\neq0$ and $x_{i_q}^{(q)}\neq0$ to reliably estimate the second-order interaction parameter $w_{i_p,i_q}^{(p,q)}$. The instances with either $x_{i_p}^{(p)}=0$ or $x_{i_q}^{(q)}=0$ cannot be used for estimating $w_{i_p,i_q}^{(p,q)}$. That is to say, on a sparse dataset where there are too few or even no cases for some higher order interactions, nonlinear SVMs are likely to degenerate into linear SVMs. Therefore, factorizing and projecting higher order interactions into a consistent latent space would facilitate parameter estimation under sparsity.

\begin{figure}
\centering
    \begin{minipage}[l]{\columnwidth}
      \centering
      \includegraphics[width=1\textwidth]{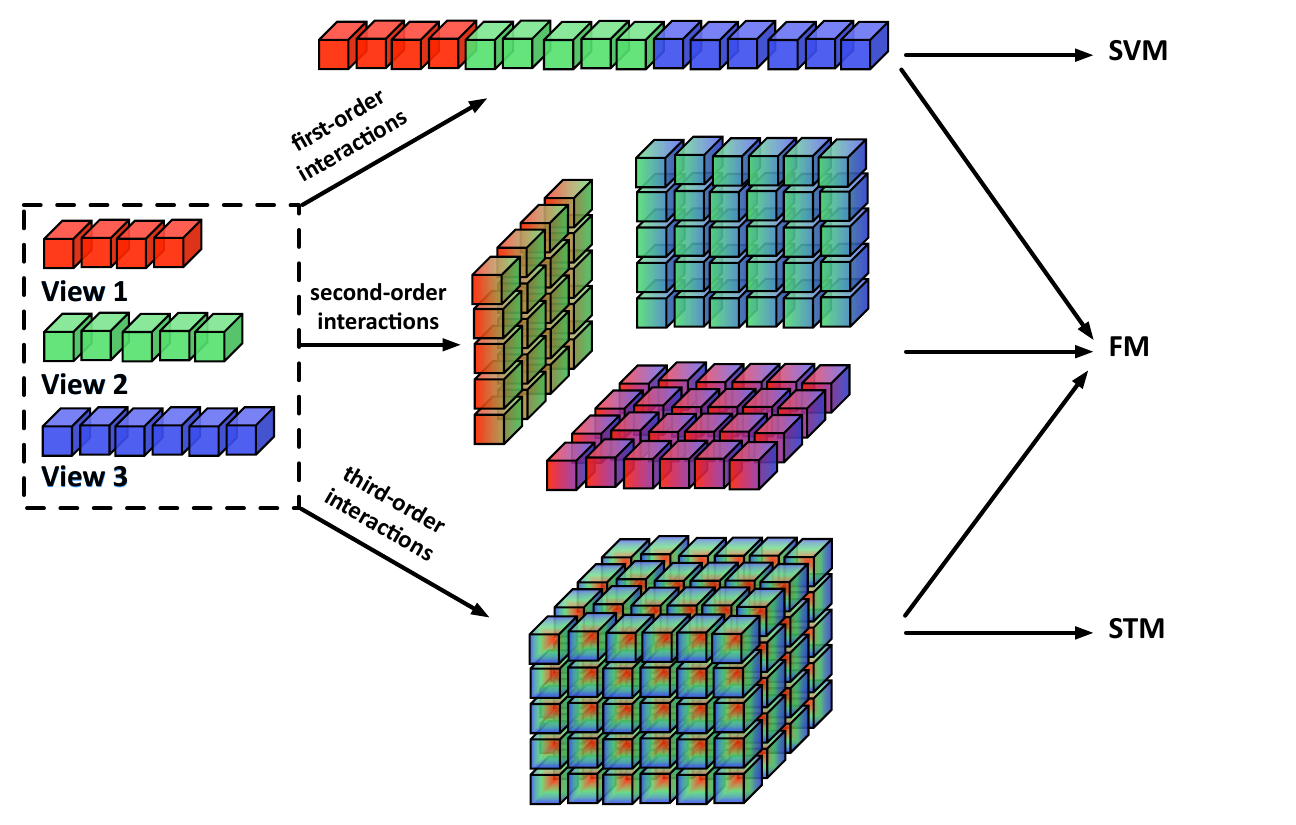}
    \end{minipage}
  \caption{Related work (and variations) on modeling feature interactions in multi-view data.}\label{fig:related}
\end{figure}

\subsection{STM Model}\label{sec:stm}

Cao et al.~investigated multi-view classification by modeling features interactions across views as a tensor, {\em i.e.}, $\mathcal{X}=\mathbf{x}^{(1)}\circ\cdots\circ\mathbf{x}^{(m)} \in \mathbb{R}^{I_1\times\cdots\times I_m}$ \cite{cao2014tensor} and solved the problem in the framework of support tensor machines (STMs) \cite{tao2005supervised}. Basically, as shown in Figure~\ref{fig:related}, only the highest-order interactions are explored:
\begin{align}
\hat{y}=\sum_{i_1=1}^{I_1}\cdots\sum_{i_m=1}^{I_m} w_{i_1,...,i_m} \left(\prod_{p=1}^m x_{i_p}^{(p)}\right)\label{eq:cao}
\end{align}
where $w_{i_1,...,i_m}=\prod_{p=1}^m w_{i_p}^{(p)}$, {\em i.e.}, a rank-one decomposition of the weight tensor $\mathcal{W} \in \mathbb{R}^{I_1\times\cdots\times I_m}$ \cite{cao2014tensor}.

However, estimating a lower order interaction ({\em e.g.}, a pairwise one) reliably is easier than estimating a higher order one, and lower order interactions can usually explain the data sufficiently \cite{rendle2010pairwise,cai2011low}. Thus, it motivates us to include the lower order interactions into the model. Moreover, instead of a rank-one decomposition, it is desirable to apply a higher rank decomposition of the weight tensor to capture more latent factors and thereby achieving a better approximation to the original interaction parameters.

\subsection{FM Model}\label{sec:fm}

Rendle introduced factorization machines (FMs) \cite{rendle2010factorization} that combine the advantages of SVMs with factorization models. The model equation of a 2-way FM is as follows:
\begin{align}
\hat{y}=w_0+\sum_{i=1}^d w_i x_i+\sum_{i=1}^d\sum_{j=i+1}^d \left\langle \mathbf{v}_i, \mathbf{v}_j \right\rangle x_i x_j \label{eq:fm}
\end{align}
where $d=\sum_{p=1}^{m} I_p$ and $\left\langle \mathbf{v}_i, \mathbf{v}_j \right\rangle=\sum_{f=1}^k v_{i,f} v_{j,f}$.

Note that pairwise interactions between all features are included in FMs without consideration of the view segmentation. In the multi-view setting, there can be redundant correlations between features within the same view, {\em i.e.}, intra-view correlations, which are thereby unworthy of consideration. Field-aware FMs \cite{libffm} integrate the field/view concept into the FM model where the extension is limited to the second-order feature interactions. The coupled group lasso model \cite{yan2014coupled} is essentially an application of the 2-way FMs in multi-view classification. Let {\bsfm} denote the multi-view variation of FMs with a decision function as follows:
\begin{align}
\hat{y}
& = w_0+\sum_{p=1}^m\sum_{i_p=1}^{I_p} w_{i_p}^{(p)} x_{i_p}^{(p)} + \sum_{i_1=1}^{I_1}\sum_{i_2=1}^{I_2} \left\langle \mathbf{v}_{i_1}^{(1)}, \mathbf{v}_{i_2}^{(2)} \right\rangle x_{i_1}^{(1)} x_{i_2}^{(2)} \nonumber \\
& + \cdots + \sum_{i_{m-1}=1}^{I_{m-1}}\sum_{i_m=1}^{I_m} \left\langle \mathbf{v}_{i_{m-1}}^{(m-1)}, \mathbf{v}_{i_m}^{(m)} \right\rangle x_{i_{m-1}}^{(m-1)} x_{i_m}^{(m)} \label{eq:fm2}
\end{align}

The pairwise interaction parameter $w_{i_p,i_q}^{(p,q)}=\left\langle \mathbf{v}_{i_p}^{(p)}, \mathbf{v}_{i_q}^{(q)} \right\rangle$ in Eq.~(\ref{eq:fm2}) indicates that $w_{i_p,i_q}^{(p,q)}$ can be learned from instances with $x_{i_p}^{(p)}\neq0$ and some $x_{i_{q'}}^{(q')}\neq0$ (sharing $\mathbf{v}_p$), or $x_{i_q}^{(q)}\neq0$ and some $x_{i_{p'}}^{(p')}\neq0$ (sharing $\mathbf{v}_q$). It makes {\bsfm} (and FMs) more effective under sparsity than SVMs where only instances with $x_{i_p}^{(p)}\neq0$ and $x_{i_q}^{(q)}\neq0$ can be used to learn the second-order feature interaction $w_{i_p,i_q}^{(p,q)}$.

However, the interaction parameters for different orders are completely independent in {\bsfm} (and FMs), {\em e.g.}, the first-order interaction parameter, $w_{i_p}^{(p)}$, and the second-order interaction parameter, $\mathbf{v}_{i_p}^{(p)}$, in Eq.~(\ref{eq:fm2}). Furthermore, as illustrated in Figure~\ref{fig:related}, additional sets of model parameters will be introduced when we consider higher order feature interactions in {\bsfm} (and FMs) which makes the learning process harder. A more effective strategy is needed when including the higher order interactions.

\section{Multi-view Machine Model}\label{sec:method}

\subsection{Model Formulation}

The key challenge of multi-view prediction is to model the interactions between features from different views, wherein complementary information is contained. Here, we consider nesting all interactions up to the $m$th-order between $m$ views:

{
\vspace{-5pt}
\scriptsize
\begin{align}
\hat{y}
& = \underbrace{\beta_0}_\text{global bias} + \underbrace{\sum_{p=1}^m\sum_{i_p=1}^{I_p} \beta_{i_p}^{(p)} x_{i_p}^{(p)}}_\text{first-order interactions} \nonumber \\
& + \underbrace{\sum_{i_1=1}^{I_1}\sum_{i_2=1}^{I_2} \beta_{i_1,i_2}^{(1,2)} x_{i_i}^{(1)} x_{i_2}^{(2)} + \cdots + \sum_{i_{m-1}=1}^{I_{m-1}}\sum_{i_m=1}^{I_m} \beta_{i_{m-1},i_m}^{(m-1,m)} x_{i_{m-1}}^{(m-1)} x_{i_m}^{(m)}}_\text{second-order interactions} \nonumber \\
& + \cdots + \underbrace{\sum_{i_1=1}^{I_1}\cdots\sum_{i_m=1}^{I_m} \beta_{i_1,...,i_m} \left(\prod_{p=1}^m x_{i_p}^{(p)}\right)}_\text{$m$th-order interactions} \label{eq:nest}
\end{align}
}

Let us add an extra feature with constant value $1$ to the feature vector $\mathbf{x}^{(p)}$, {\em i.e.}, ${\mathbf{z}^{(p)}}^T = ({\mathbf{x}^{(p)}}^T,1) \in \mathbb{R}^{I_{p}+1}, \forall p=1,...,m$. Then, Eq.~(\ref{eq:nest}) can be compactly rewritten as:
\begin{align}
\hat{y}=\sum_{i_1=1}^{I_1+1}\cdots\sum_{i_m=1}^{I_m+1} w_{i_1,...,i_m} \left(\prod_{p=1}^m z_{i_p}^{(p)}\right)\label{eq:nest2}
\end{align}
where $w_{I_1+1,...,I_m+1} = \beta_0$ and $w_{i_1,...,i_m} = \beta_{i_1,...,i_m}, \forall i_p \le I_p$. For $w_{i_1,...,i_m}$ with some indexes satisfying $i_p=I_p+1$, it encodes lower order interaction between views whose $i_p \le I_p$. Hereinafter, let $w_{i_p}^{(p)}$ denote $w_{i_1,...,i_m}$ where only $i_p \le I_p$ and $i_q=I_q+1, q\neq p$, and let $w_{i_p,i_q}^{(p,q)}$ denote $w_{i_1,...,i_m}$ where $i_p \le I_p$, $i_q \le I_q$ and $i_r=I_r+1, r \notin \{p,q\}$, {\em etc.}

\begin{figure}
\centering
    \begin{minipage}[l]{\columnwidth}
      \centering
      \includegraphics[width=1\textwidth]{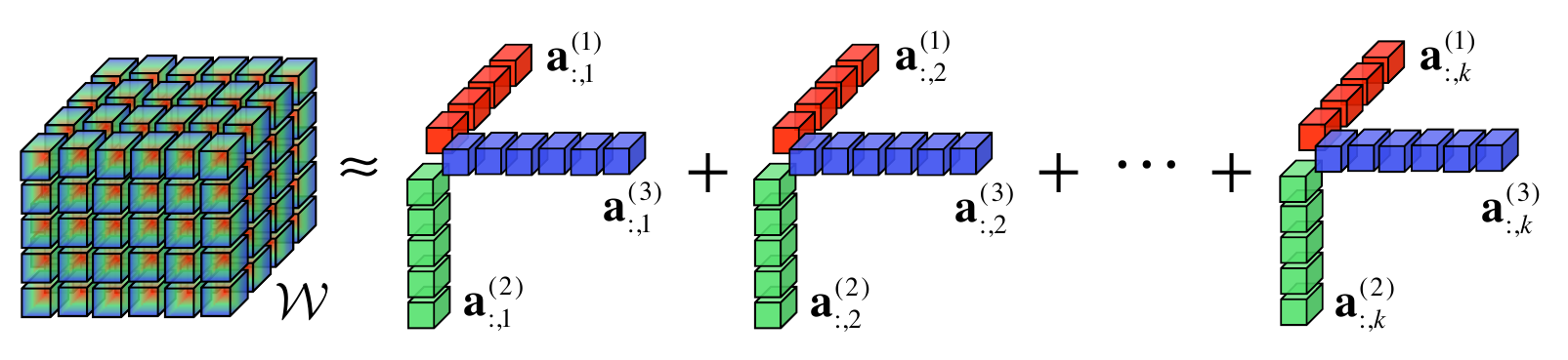}
    \end{minipage}
  \caption{CP factorization. The third-order ($m=3$) tensor $\mathcal{W}$ is approximated by $k$ rank-one tensors. The $f$-th factor tensor is the tensor product of three vectors, {\em i.e.}, $\mathbf{a}^{(1)}_{:,f}\circ\mathbf{a}^{(2)}_{:,f}\circ\mathbf{a}^{(3)}_{:,f}$.}\label{fig:cp}
\end{figure}

The number of parameters in Eq.~(\ref{eq:nest2}) is $\prod_{p=1}^m (I_p+1)$, which can make the model prone to overfitting and ineffective on sparse data. Therefore, we assume that the effect of interactions has a low rank and the $m$th-order weight tensor $\mathcal{W}=\{w_{i_1,...,i_m}\} \in \mathbb{R}^{(I_1+1)\times\cdots\times(I_m+1)}$ can be factorized into $k$ factors:
\begin{align}
\mathcal{W}=\mathbf{C}\times_1\mathbf{A}^{(1)}\times_2\cdots\times_m\mathbf{A}^{(m)} \label{eq:cp}
\end{align}
where $\mathbf{A}^{(p)} \in \mathbb{R}^{(I_p+1) \times k}$, and $\mathbf{C} \in \mathbb{R}^{k \times \cdots \times k}$ is the identity tensor, {\em i.e.}, $c_{i_1,...,i_m}=\delta(i_1=\cdots=i_m)$. Basically, Eq.~(\ref{eq:cp}) is a CANDECOMP/PARAFAC (CP) factorization \cite{kolda2009tensor} as shown in Figure~\ref{fig:cp}, with element-wise notation $w_{i_1,...,i_m}=\sum_{f=1}^k\prod_{p=1}^m a_{i_p,f}^{(p)}$. The number of model parameters is reduced to $k\sum_{p=1}^m (I_p+1)=k(m+d)$. It transforms Eq.~(\ref{eq:nest2}) into:
\begin{align}
\hat{y}=\sum_{i_1=1}^{I_1+1}\cdots\sum_{i_m=1}^{I_m+1} \left(\prod_{p=1}^m z_{i_p}^{(p)}\right) \left(\sum_{f=1}^k\prod_{p=1}^m a_{i_p,f}^{(p)}\right)\label{eq:mvm}
\end{align}

\begin{figure*}
\centering
    \begin{minipage}[l]{1.6\columnwidth}
      \centering
      \includegraphics[width=1\textwidth]{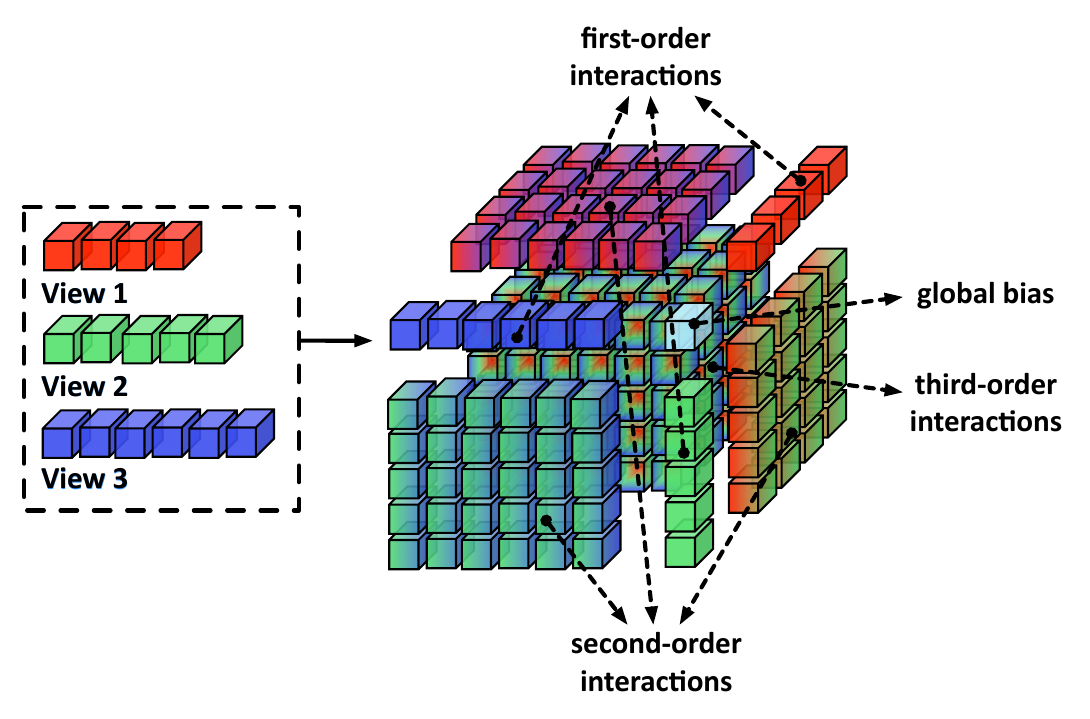}
    \end{minipage}
  \caption{Multi-view machines. The full-order feature interactions in multi-view data are modeled in a tensor and jointly factorized into a common latent subspace.}\label{fig:mvm}
\end{figure*}

We name this model in Eq.~(\ref{eq:mvm}) as multi-view machines (MVMs). As shown in Figure~\ref{fig:mvm}, the full-order interactions between multiple views are modeled in a tensor, and they are factorized collectively. The model parameters that have to be estimated are:
\begin{align}
\mathbf{A}^{(p)} \in \mathbb{R}^{(I_p+1) \times k},~p=1,...,m \label{eq:para}
\end{align}
where the $i_p$-th row ${\mathbf{a}_{i_p}^{(p)}}^T=(a_{i_p,1}^{(p)},...,a_{i_p,k}^{(p)})$ within $\mathbf{A}^{(p)}$ describes the $i_p$-th feature in the $p$-the view with $k$ factors. 

\begin{definition}[Bias Factor]
The bias factor is a collection of bias from each factor. In MVMs, the last row of $\mathbf{A}^{(p)}$, i.e., ${\mathbf{a}_{I_p+1}^{(p)}}^T$, represents the bias factor of the $p$-th view, and it is always associated with $z_{I_p+1}^{(p)}=1$ in Eq.~(\ref{eq:mvm}). 
\end{definition}

Hence,
\begin{align}
w_{I_1+1,...,I_m+1} = \sum_{f=1}^k\prod_{p=1}^m a_{I_p+1,f}^{(p)} \label{eq:w0}
\end{align}
is the {\em global bias}, denoted as $w_0$ hereinafter.

\subsection{Time Complexity}

Next, we show how to compute the decision function of MVMs efficiently. The straightforward time complexity of Eq.~(\ref{eq:mvm}) is $O(k\prod_{p=1}^m (I_p+1))$. However, we observe that there is no model parameter which directly depends on feature interactions, due to the joint factorization. Therefore, the time complexity can be largely reduced.

\begin{lemma}
The model equation of MVMs can be computed in $O(k(m+d))$.
\end{lemma}

\begin{proof}
The feature interactions in Eq.~(\ref{eq:mvm}) can be reformulated as:
\begin{align}
  & \sum_{i_1=1}^{I_1+1}\cdots\sum_{i_m=1}^{I_m+1} \left(\prod_{p=1}^m z_{i_p}^{(p)}\right) \left(\sum_{f=1}^k\prod_{p=1}^m a_{i_p,f}^{(p)}\right) \nonumber \\
= & \sum_{f=1}^k\sum_{i_1=1}^{I_1+1}\cdots\sum_{i_m=1}^{I_m+1} \left(\prod_{p=1}^m z_{i_p}^{(p)} a_{i_p,f}^{(p)}\right) \nonumber \\
= & \sum_{f=1}^k\left(\sum_{i_1=1}^{I_1+1} z_{i_1}^{(1)} a_{i_1,f}^{(1)}\right) \cdots \left(\sum_{i_m=1}^{I_m+1} z_{i_m}^{(m)} a_{i_m,f}^{(m)}\right) \label{eq:complexity}
\end{align}

This equation has only linear complexity in both $k$ and $I_p$. Thus, its time complexity is $O(k(m+d))$, which is in the same order of the number of model parameters.
\end{proof}

\subsection{Discussion}

The joint factorization of the global bias and the full-order interactions is important for MVMs. Thus, dependencies exist when interactions share the same feature. It benefits MVMs for parameter estimation under sparsity, since the latent factor $\mathbf{a}_{i_p}^{(p)}$ can be learned from any instances whose $x_{i_p}^{(p)}\neq0$, which allows the second-order interaction $w_{i_p,i_q}^{(p,q)}$ can be approximated from instances whose $x_{i_p}^{(p)}\neq0$ {\em or} $x_{i_q}^{(q)}\neq0$ rather than instances whose $x_{i_p}^{(p)}\neq0$ {\em and} $x_{i_q}^{(q)}\neq0$ as in nonlinear SVMs. Therefore, the interaction parameters in MVMs can be effectively learned without direct observations of such interactions in a training set of sparse data.

The main difference between FMs and MVMs is that the interaction parameters for different orders are completely independent in FMs, {\em e.g.}, the first-order interaction $w_{i_p}^{(p)}$ and the second-order interaction $\mathbf{v}_{i_p}^{(p)}$ in Eq.~(\ref{eq:fm2}). On the contrary, in MVMs, all orders of interactions share the same set of latent factors $\mathbf{a}_{i_p}^{(p)}$ in Eq.~(\ref{eq:mvm}). For example, the combination of $\mathbf{a}_{i_p}^{(p)}$ and the bias factors from other $m-1$ views, {\em i.e.}, $\mathbf{a}_{I_1+1}^{(1)},...,\mathbf{a}_{I_{p-1}+1}^{(p-1)},\mathbf{a}_{I_{p+1}+1}^{(p+1)},...,\mathbf{a}_{I_m+1}^{(m)}$, approximates the first-order interaction $w_{i_p}^{(p)}$. Similarly, we can obtain the second-order interaction $w_{i_p,i_q}^{(p,q)}$ by combining $\mathbf{a}_{i_p}^{(p)}$, $\mathbf{a}_{i_q}^{(q)}$ and other $m-2$ bias factors. Therefore, compared to MVMs, FMs are partially and independently factorized. Such difference is more significant for higher order FMs. As summarized in Table~\ref{tab:baseline}, assuming the same number of factors for different orders of interactions, the model complexity of an $m$-way FM is $O(kmd)$ which can be much larger than $O(k(m+d))$ of MVMs.

\begin{table*}
\centering
\caption{Summary of related work. Model complexity refers to both the number of parameters in the model and the time complexity to compute the decision function.}
\label{tab:baseline}
\begin{tabular}{lccc}
\toprule
Method & Model complexity & Feature interactions & Parameter factorization \\
\midrule
Support vector machines (SVMs) \cite{vapnik2000nature} & $O(d)$ & first-order & none \\
Support tensor machines (STMs) \cite{tao2005supervised} & $O(kd)$ & highest-order & factorized ($k=1$ \cite{cao2014tensor}) \\
Factorization machines (FMs) \cite{rendle2010factorization} & $O(kmd)$ & up to full-order & partially and independently factorized \\
Multi-view machines (MVMs) & $O(k(m+d))$ & full-order & fully and jointly factorized \\
\bottomrule
\end{tabular}
\end{table*}

\subsection{Extensions}

MVMs are flexible in the interactions of interests. That is to say, when there are too many views available for a learning task and interactions between some of them may obviously be physically meaningless, or sometimes the very high order interactions may not be intuitively interpretable, it is not desirable to include these potentially redundant interactions in the model. In such scenarios, one can (1) partition (overlapping) groups of views, (2) construct multiple MVMs on these view groups where the full-order interactions within each group are included, and (3) implement a coupled matrix/tensor factorization \cite{hong2013co}. This strategy excludes the inter-group feature interactions.

On the other hand, in scenarios where the view segmentation is not given, one may be aggressive to consider interactions between all features, which becomes the problem setting of the original FMs. To achieve this purpose, we can simply repeat the same feature set in multiple views. Overall, MVMs are applicable with either conservative or radical strategies. Although MVMs can be easily adapted to include/exclude interactions between any features, that is outside the scope of this paper; our focus is on investigating how to effectively explore the full-order feature interactions from a given set of views.

\section{Learning Multi-view Machines}\label{sec:learning}

To learn model parameters in MVMs, we consider the following regularization framework:
\begin{align}
\operatornamewithlimits{argmin}_{\Theta}\sum_{(\mathbf{x},y)\in\mathcal{D}}\mathcal{L}(\hat{y}(\mathbf{x}|\Theta),y)+\lambda\Omega(\Theta)
\end{align}
where $\Theta=\{\mathbf{A}^{(p)}|~p=1,...,m\}$ represents all model parameters, $\mathcal{L}(\cdot)$ is the loss function, $\Omega(\cdot)$ is the regularization term, and $\lambda$ is the trade-off between the empirical loss and the risk of overfitting.

Importantly, MVMs can be used to perform a variety of machine learning tasks, depending on the choices of the loss function. For example, to conduct regression, one can use the square error:
\begin{align}
\mathcal{L}^S(\hat{y}(\mathbf{x}|\Theta),y)=(\hat{y}(\mathbf{x}|\Theta)-y)^2 \label{eq:square}
\end{align}
and for classification problems, the logit loss can be used:
\begin{align}
\mathcal{L}^L(\hat{y}(\mathbf{x}|\Theta),y)=\log(1+\exp(-y\cdot\hat{y}(\mathbf{x}|\Theta))) \label{eq:logit}
\end{align}
or the hinge loss.
The regularization term is chosen based on our prior knowledge about the model parameters. Typically, we can apply L2-norm.

\subsection{Gradient Descent}

The model can be learned efficiently by alternating least square (ALS), stochastic gradient descent (SGD), L-BFGS, {\em etc.} From Eq.~(\ref{eq:complexity}), the gradient of the MVM model is:
\begin{align}
\frac{\partial\hat{y}(\mathbf{x}|\Theta)}{\partial\theta} = 
& z_{i_p}^{(p)} \left(\sum_{i_1=1}^{I_1+1} z_{i_1}^{(1)} a_{i_1,f}^{(1)}\right) \cdots \left(\sum_{i_{p-1}=1}^{I_{p-1}+1} z_{i_{p-1}}^{(p-1)} a_{i_{p-1},f}^{(p-1)}\right) \nonumber \\
& \left(\sum_{i_{p+1}=1}^{I_{p+1}+1} z_{i_{p+1}}^{(p+1)} a_{i_{p+1},f}^{(p+1)}\right) \cdots \left(\sum_{i_m=1}^{I_m+1} z_{i_m}^{(m)} a_{i_m,f}^{(m)}\right) \label{eq:gradient}
\end{align}
where $\theta=a_{i_p,f}^{(p)}$, and $z_{i_p}^{(p)}=1$ if $i_p=I_p+1$, otherwise $z_{i_p}^{(p)}=x_{i_p}^{(p)}$. It validates that MVMs possess the multilinearity property \cite{rendle2012factorization}, because the gradient along $\theta$ is independent of the value of $\theta$ itself.

Note that in Eq.~(\ref{eq:gradient}), the sum $\sum_{i_p=1}^{I_p+1} z_{i_p}^{(p)} a_{i_p,f}^{(p)}$ and their product can be precomputed for updating the $f$-th factor of all features. Hence, each gradient can be computed in constant time $O(1)$. In an iteration, including the precomputation time, all parameters can be updated in $O(k(m+d))$. It can be even reduced under sparsity, where most of the elements in $\mathbf{x}$ (or $\mathbf{z}$) are $0$ and thus, the sums have only to be computed over the non-zero elements, and only non-zero parameters need to be updated according to Eq.~(\ref{eq:gradient}).

It is straightforward to embed Eq.~(\ref{eq:gradient}) into the gradient of the loss functions {\em e.g.}, Eqs.~(\ref{eq:square})-(\ref{eq:logit}), for direct optimization, as follows:
\begin{align}
\frac{\partial\mathcal{L}^S(\hat{y}(\mathbf{x}|\Theta),y)}{\partial\theta} = 2(\hat{y}(\mathbf{x}|\Theta)-y) \cdot \frac{\partial\hat{y}(\mathbf{x}|\Theta)}{\partial\theta} \label{eq:grad_square}
\end{align}
\begin{align}
\frac{\partial\mathcal{L}^L(\hat{y}(\mathbf{x}|\Theta),y)}{\partial\theta} = \frac{-y}{1+\exp(y\cdot\hat{y}(\mathbf{x}|\Theta))} \cdot \frac{\partial\hat{y}(\mathbf{x}|\Theta)}{\partial\theta} \label{eq:grad_logit}
\end{align}

\subsection{Distributed Implementation}\label{sec:spark}

Web-scale applications in the real world always contain a huge number of entities represented in multiple views, {\em e.g.}, users, movies, ads, queries, with millions of instances, {\em e.g.}, ratings, impressions. In this section, we introduce a design for scalable learning and its implementation on top of GraphX \cite{graphx}, which is a component of Spark \cite{spark} for graphs and graph-parallel computation and provides high performance, scalability and fault-tolerance for the learning process.

The training data is represented as a graph that contains two types of vertices, {\em i.e.}, instance vertices and feature vertices. A directed edge from a feature vertex to an instance vertex exists if the feature is non-zero in the instance. The graph representation is efficient due to the inherent sparsity of the training data. The factor vector (or weight coefficient that is not factorized in some baselines) of a feature is represented as attributes of the corresponding feature vertex, the label information of an instance is represented as the attribute of the corresponding instance vertex, and the feature value is represented as the edge attribute. For distributed learning, the graph is partitioned and scheduled to different computing nodes for execution by the underlying distributed graph framework. In this manner, both data parallelism and model parallelism are achieved.

Each iteration in the gradient descent algorithm consists of two major steps, {\em i.e.}, feed-forward and back-propagation. In the feed-forward process, messages are sent from feature vertices to instance vertices following the edges which are arrays $\mathbf{b}=\mathbb{R}^k$ where $b_f=z_{i_p}^{(p)}*a_{i_p,f}^{(p)}$. An instance vertex receives all messages from its connected feature vertices and sums them in view-wise. The predicted value is then computed accordingly based on Eq.~(\ref{eq:complexity}). In the back-propagation process, messages are sent from instance vertices to feature vertices which are arrays $\mathbf{c}=\mathbb{R}^k$ where each element represents a gradient. A feature vertex averages the gradients received from its connected instance vertices and updates the factor vector accordingly based on Eqs.~(\ref{eq:grad_square})-(\ref{eq:grad_logit}).

\section{Experiments}\label{sec:exp}


\subsection{Experimental Setup}\label{sec:setup}

\noindent\textbf{Data collections.}
To evaluate the performance of multi-view prediction, we conduct extensive experiments on the {\ml} dataset for movie rating prediction (regression) and the {\bing} dataset for CTR prediction (classification), respectively.
\begin{itemize}
\item {\bf\ml~dataset}\footnote{\url{http://grouplens.org/datasets/movielens}}. A regression task for rating prediction is studied on the public dataset, {\ml}. Ratings are made on a 5-star scale, with half-star increments. Each rating in this dataset has three views, {\em i.e.}, users, movies and implicit user feedback. The user view consists of binary feature vectors for user ids, and thus for each rating there is only one non-zero feature in the user view, {\em i.e.}, the associated user id; the same for the movie view. The implicit feedback view is constructed following SVD++ \cite{koren2008factorization} to capture users' history information. Specifically, it consists of all movies the user has ever rated and it is normalized. Hence, this view makes use of implicit feedback information and indicates users' preference. For this problem, the performance is measured by root mean square error (RMSE).
\item {\bf\bing~dataset}\footnote{The dataset is used internally in the Bing Ads team for model experiments rather than training product models.}. A classification task for CTR prediction is investigated on a dataset collected from ad impression logs of Bing, comprising three views: queries, ad URLs and impression information. Each instance is labeled as 1 if the impression is clicked and -1 otherwise. The query view consists of unigrams of user query words\footnote{Stemming, lemmatization, removing stop-words, {\em etc.}, are handled beforehand.}. The ad URL view includes URLs corresponding to the shown ads. The impression view is composed of impression locations and matched types. All features are hashed as integer ids and represented by binary values. There are multiple non-zero features in the query view, only one non-zero feature in the ad URL view, and 2 non-zero features in the impression view. Area under the curve (AUC) is used as the evaluation metric.
\end{itemize}

See Table~\ref{tab:dataset} for more information about the statistics and parameters used for each dataset.

\begin{table}[t]
\centering
\caption{The statistics and parameters for each dataset. The number in braces indicates the dimensionality of the corresponding view.}
\label{tab:dataset}
\begin{tabular}{lcc}
\toprule
Dataset &{\ml} &{\bing} \\
\midrule
            &users (138,493)            &queries (958,426)      \\
Views       &movies (27,278)            &ad URLs (1,935,510)    \\
            &impl. (27,278)             &impressions (18)       \\
$n$         &20,000,263                 &28,622,281             \\
$\eta$      &0.1                        &0.1                    \\
$\lambda$   &0.01                       &0.01                   \\
$k$         &20                         &20                     \\
\#iterations&200                        &200                    \\
\bottomrule
\end{tabular}
\end{table}


\noindent\textbf{Compared models.}
In order to demonstrate the effectiveness of modeling feature interactions in multi-view data, we compare the following models:
\begin{itemize}
\item \textbf{Linear regression/logistic regression (LR)}. We implement linear regression for regression tasks, {\em e.g.}, rating prediction, and logistic regression for classification tasks, {\em e.g.}, CTR prediction. They are essentially representative linear models (including linear SVMs), but with different loss functions, {\em e.g.}, the square error and the logit loss, respectively. It is discussed in the form of SVMs in detail in Section~\ref{sec:svm}.
\item \textbf{Tensor factorization (TF)} is a generalization of matrix factorization to higher orders. We can directly use tensors to model the multi-view data and factorize the weight tensor \cite{cao2014tensor}. When the hinge loss is used, it can be solved in the framework of support tensor machines (STMs) \cite{tao2005supervised}. When there are two views with categorical features, TF is reduced to conventional matrix factorization without bias terms. It is introduced as the STM model in Section~\ref{sec:stm}.
\item \textbf{Factorization machine (FM)} explores pairwise interactions between all features without consideration of the view segmentation \cite{rendle2010factorization}. Its adaptation in the multi-view setting, denoted as {\bsfm}, considers feature interactions across views with the decision function in Eq.~(\ref{eq:fm2}). This FM variation is specifically reviewed in Section~\ref{sec:fm}. In addition to the popular 2-way FM model, we also implemented 3-way FMs to include higher order interactions, denoted as {\threewayfm}, where feature interactions with different orders are modeled but with separate sets of parameters\footnote{In experiments, the rank $k$ in {\threewayfm} is set to 10 for both the second-order and the third-order interactions, so that the number of model parameters stays the same as other factorization baselines.}. Moreover, we regularized the second-order and the third-order interactions sharing the same latent factors and assigned the global bias and the first-order interactions with independent weights that are not factorized, denoted as {\partialmvm}.
\item \textbf{Multi-view machine (MVM)} is our proposed model to explore the full-order interactions embedded within multi-view data, where feature interactions with different orders are jointly factorized and thereby sharing the same set of latent factors.
\end{itemize}


\noindent\textbf{Configuration.} All compared models are implemented on top of GraphX in Spark and trained with iterative forward and backward steps as in introduced in Section~\ref{sec:spark}. The stochastic gradient descent with adaptive (sub)gradient \cite{duchi2011adaptive} is used as the optimization method. The code has been made available at GitHub\footnote{\url{https://github.com/cloudml/zen/tree/mvm_opt/ml/src/main/scala/com/github/cloudml/zen/ml/recommendation}}.

For a fair comparison, the same parameter setting in Table~\ref{tab:dataset} is used for all compared models. A deterministic data sampling is applied on both datasets so that 80\% data is used for training and the other 20\% for test. All models are trained with the same hardware configuration, where 10 homogeneous computing nodes are connected via 40Gbps Infiniband network and each node has 16 2.40GHz Intel(R) Xeon(R) CPU E5-2665 cores and 128GB memory. There are 1 driver configured with 25GB memory and 10 workers configured with 100GB memory. The data is partitioned into 160 partitions based on node degree \cite{xie2014distributed} to balance the load in each core and reduce the communication among cores.

\subsection{Multi-view Prediction Accuracy}\label{sec:performance}

\begin{table}
\centering
\caption{Prediction accuracy. {\color{red}{$\downarrow$}} indicates the smaller the value the better the performance; {\color{green}{$\uparrow$}} indicates the larger the value the better the performance.}
\label{tab:exp}
\begin{tabular}{lcc}
\toprule
Dataset &{\ml} (RMSE) \color{red}{$\downarrow$} &{\bing} (AUC) \color{green}{$\uparrow$} \\
\midrule
\mvm        &0.8376      &0.7917            \\
\fm         &0.8681      &0.7872            \\
\bsfm       &0.8447      &0.7729            \\
\threewayfm &0.9060      &0.7201            \\
\partialmvm &0.9807      &0.6947            \\
\tf         &0.8572      &0.6645            \\
\lr         &1.0017      &0.7450            \\
\bottomrule
\end{tabular}
\end{table}

The experimental results are shown in Table~\ref{tab:exp}. On the {\ml} dataset, the smaller RMSE, the better an algorithm. We can observe that {\lr} is a simple baseline because as a conventional linear model, it neglects any interactions between features. However, such feature interactions can be critical in the sparse data, which explains much better performance achieved by {\fm} through including pairwise feature interactions. We further find that {\bsfm} is able to outperform {\fm} by excluding intra-view correlations. In our case of movie rating prediction on SVD++ data \cite{koren2008factorization}, intra-view correlations indicate interactions between movies the user has rated before which do not have direct influence on the user's preference of the current movie. However, inter-view correlations include interactions between the current movie and those movies rated by the user which are critical by matching the latent factor of the current movie and that of rated movies in the past. It validates the importance of introducing the view concept to learn an effective model in many problems.

The two variations of {\bsfm}, {\threewayfm} and {\partialmvm}, add the third-order feature interactions in addition to the 2-way {\bsfm}. The difference is that {\partialmvm} uses the same set of latent factors for the second-order and the third-order interactions, while {\threewayfm} introduces different parameters for interactions with different orders. It seems that the inclusion of higher order interactions fails to bring us any accuracy improvement, but {\tf} manages to perform well by solely relying on the highest-order interactions. It might imply that the consensus and complementary information between lower order and higher order interactions need to be better taken care of, which leads to our MVM model. Overall, we can observe from Table~\ref{tab:exp} that MVM achieves the best performance through joint factorization of feature interactions with different orders.

On the {\bing} dataset, {\fm} shows better performance than {\bsfm} implying intra-view correlations might be important for this problem. Consider the impression view comprising 2 non-zero features for each instance, {\em i.e.}, impression location and matched type. Feature interactions between impression locations and matched types are not included in {\bsfm}, whose variations and {\tf} are even defeated by the linear model, {\lr}.

\begin{figure*}[!ht]
\centering
  \subfigure[Training loss.]{\label{fig:ml_rmse}
    \begin{minipage}[l]{0.7\columnwidth}
      \centering
      \includegraphics[width=1\textwidth]{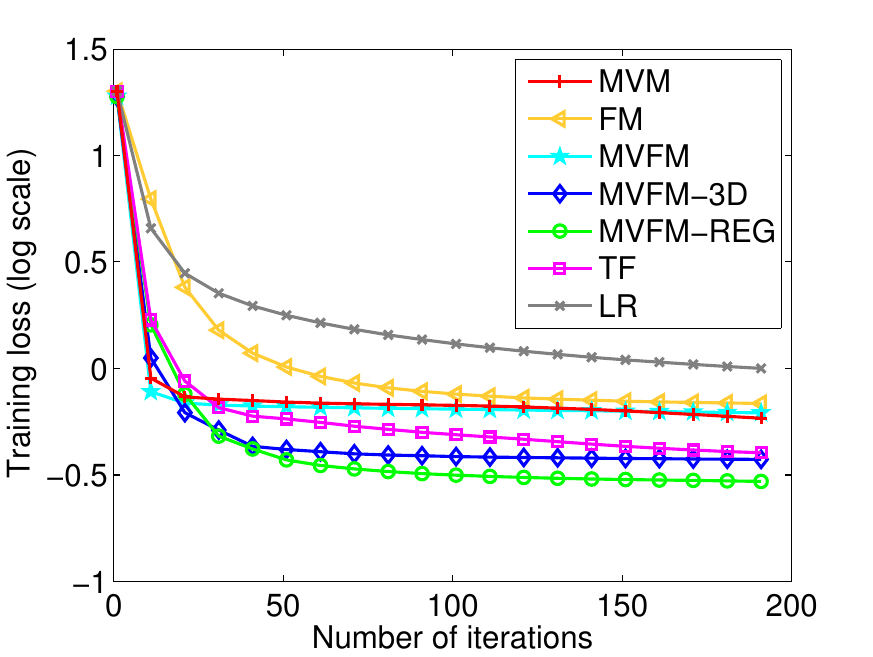}
    \end{minipage}
  }
  \subfigure[Time cost.]{\label{fig:ml_time}
    \begin{minipage}[l]{0.7\columnwidth}
      \centering
      \includegraphics[width=1\textwidth]{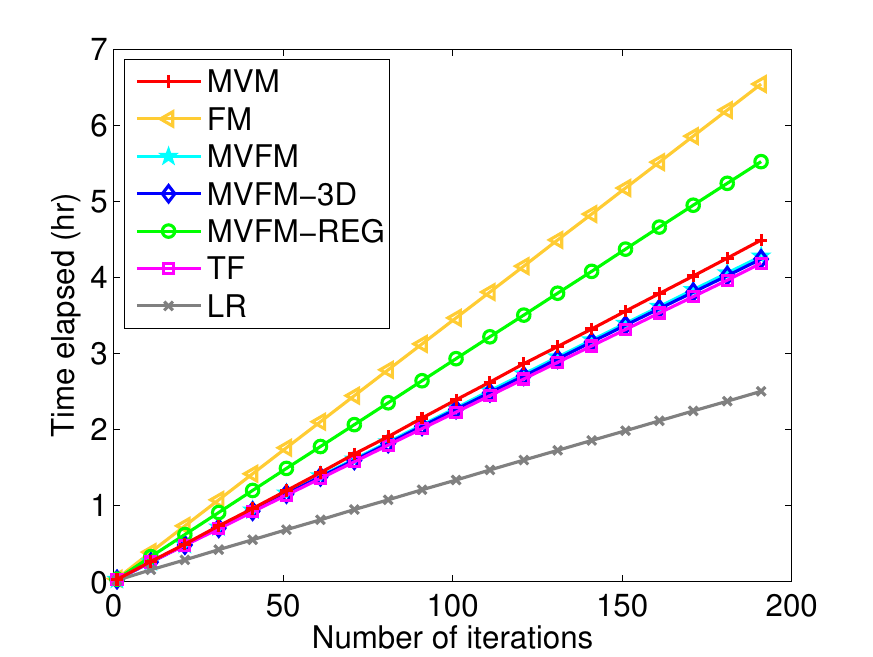}
    \end{minipage}
  }
\caption{The training procedure on the {\ml} dataset.}\label{fig:ml}
\end{figure*}

\begin{figure*}[!ht]
\centering
  \subfigure[Training loss.]{\label{fig:ads_auc}
    \begin{minipage}[l]{0.7\columnwidth}
      \centering
      \includegraphics[width=1\textwidth]{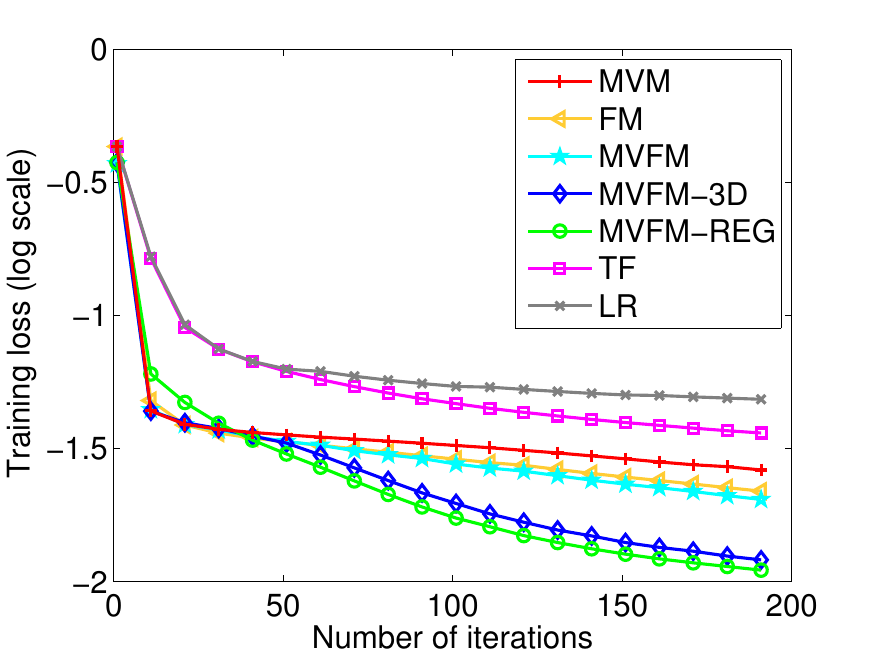}
    \end{minipage}
  }
  \subfigure[Time cost.]{\label{fig:ads_time}
    \begin{minipage}[l]{0.7\columnwidth}
      \centering
      \includegraphics[width=1\textwidth]{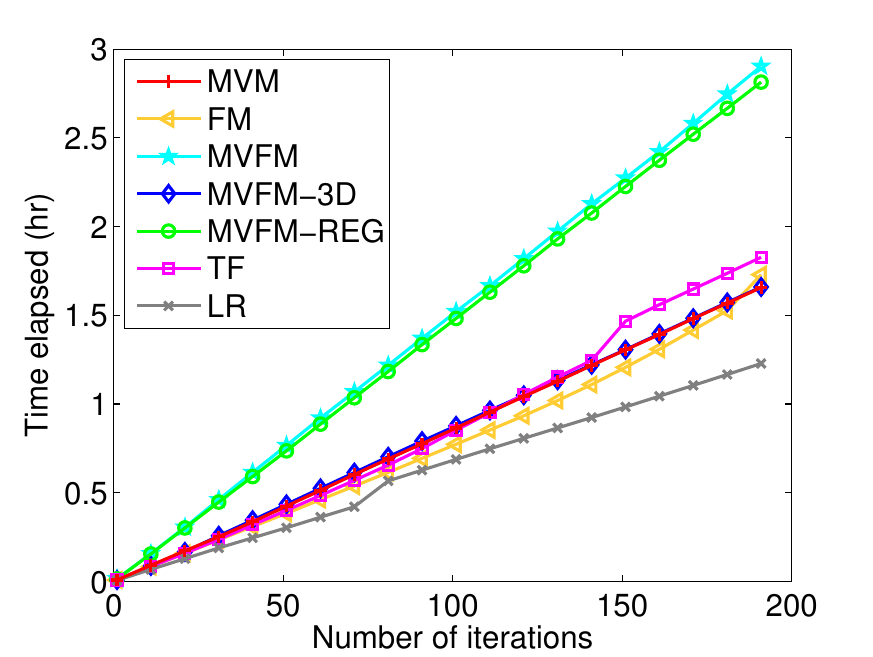}
    \end{minipage}
  }
\caption{The training procedure on the {\bing} dataset.}\label{fig:ads}
\end{figure*}

\subsection{Convergence Efficiency}\label{sec:convergence}

In this section, we show more details about the training procedure of compared models. Figure~\ref{fig:ml_rmse} illustrates the training loss on the {\ml} dataset where results are plotted in a log scale for better resolution on final convergence in the late stage. We observe that several compared models can fit themselves very well on the training data. For example, the final converged training RMSE of {\partialmvm} is as small as 0.5879; however, its test RMSE turns out to be 0.9807. It is clear that these models are easily prone to overfitting. On the contrary, our MVM model fits the training data with a moderate training RMSE = 0.7855, and achieves the best test RMSE = 0.8376, as shown in Table~\ref{tab:exp}. 

Figure~\ref{fig:ads_auc} shows similar observations on the {\bing} dataset where the overfitting problem is more significant. There is a possible reasoning about the capacity of MVMs to avoid overfitting. The joint factorization of the global bias, the first-order interactions, the second-order and higher order interactions plays a key role through the effect of bias factors $\mathbf{a}_{I_p+1}^{(p)}$. The bias factor of each view will be updated by all instances, since it is always associated with a non-zero feature. Considering lower order interactions, other factor vectors contribute to the decision function by combining with bias factors of other views, as shown in Eq.~(\ref{eq:mvm}). Therefore, bias factors are frequently retrieved and updated, making themselves relatively sensitive among model parameters. Our initial experiments found that the MVM model would suffer an unstable convergence process without the use of adaptive gradient \cite{duchi2011adaptive}. Fortunately, this problem can be greatly alleviated by adaptively choosing an appropriate learning rate, as illustrated by the monotonic convergence process shown in Figure~\ref{fig:ml_rmse} and Figure~\ref{fig:ads_auc}. With this problem solved, bias factors bring MVMs with the capacity to avoid overfitting by storing and providing the {\em global knowledge}, because each training instance will update them and each test instance will be predicted based on them. Such global knowledge is critical to model the bias information per view per factor and thus makes MVMs a robust model. In contrast, other compared methods ({\em e.g.}, {\lr}, {\fm}) use a single model parameter, {\em i.e.}, the global bias, for the purpose of the global knowledge, which is insufficient.

Figure~\ref{fig:ml_time} and Figure~\ref{fig:ads_time} compare the time cost of each model on the {\ml} dataset and the {\bing} dataset, respectively. We find that our MVM model has the best system performance among models that consider high order feature interactions without bringing too much system overhead than the linear model, {\lr}. The steep rise of {\lr} and {\tf} in Figure~\ref{fig:ads_time} appears because of fault occurrence during training and automatic recovery by Spark.

\subsection{Hyperparameter Sensitivity}\label{sec:hyper}

In all experiments, the parameter $\eta$ is heuristically set to 0.1 for MVMs and other baseline models, since the performance is insensitive to the initial learning rate by using the adaptive gradient \cite{duchi2011adaptive}. In this section, we study the influence of the other two key hyperparameters, $k$ and $\lambda$, in our MVM model. Due to the space limit, only results on {\ml} dataset are presented.

Experiments are conducted for different $k$ and the results are shown in Figure~\ref{fig:rank}. In contrast to findings in other related models based on latent factors \cite{yan2014coupled,rendle2012factorization} where accuracy can steadily get improved with larger $k$, we observe that both the converged training loss and test loss turn out to be better with the increasing of $k$ and reach a peak at $k=40$, after which the accuracy will obviously decrease. It is reasonable in a general sense, because larger $k$ renders the model with greater expressiveness which also leads to higher risk of overfitting. When the expressiveness of the model exceeds the information embedded in data, it is likely that the model will fit the training data very well yet fail on the test data. On the other hand, larger $k$ leads to more model parameters which make it harder to learn an effective model within limited iterations. In general, Figure~\ref{fig:rank} indicates that the performance of our MVM model in Table~\ref{tab:exp} can be further improved with $k=40$ at the cost of a linear increase in both runtime and memory.

Moreover, we investigate the influence of the regularization parameter $\lambda$ and present the results in Figure~\ref{fig:lambda}. We observe that our MVM model is insensitive to $\lambda$ in a relatively large range and performs well and steadily when $\lambda\le0.1$. It makes sense because large $\lambda$ will let the regularization term override the effect of the loss function and thus dominate the objective.

\begin{figure*}
\centering
  \subfigure[Influence of $k$.]{\label{fig:rank}
    \begin{minipage}[l]{0.6\columnwidth}
      \centering
      \includegraphics[width=1\textwidth]{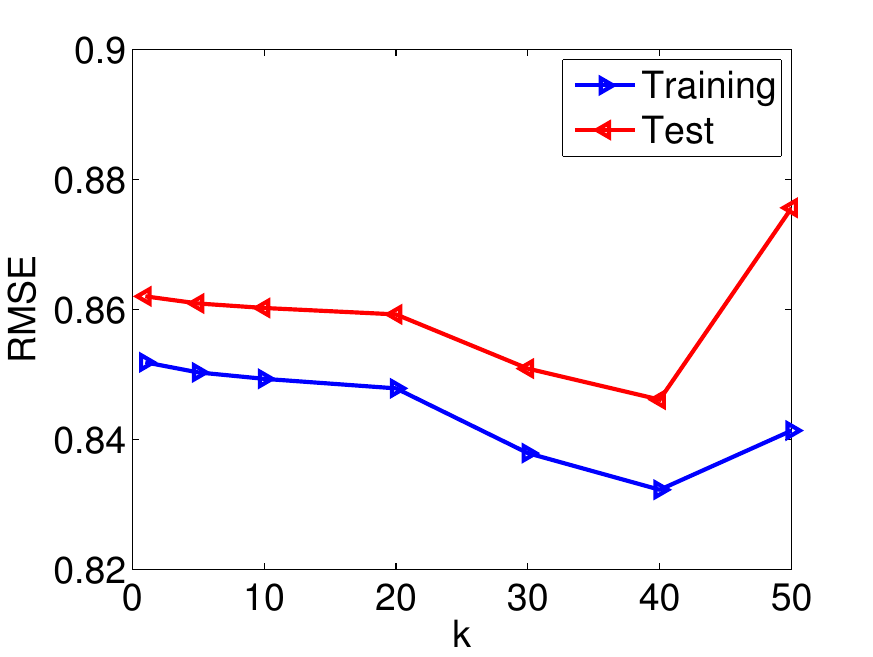}
    \end{minipage}
  }
  \subfigure[Influence of $\lambda$.]{\label{fig:lambda}
    \begin{minipage}[l]{0.6\columnwidth}
      \centering
      \includegraphics[width=1\textwidth]{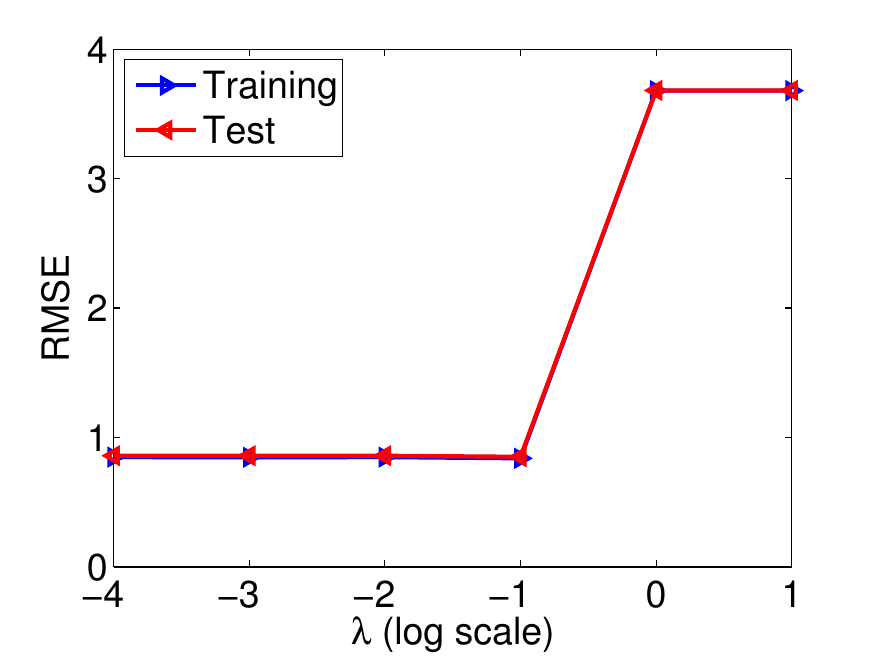}
    \end{minipage}
  }
  \subfigure[Speedup.]{\label{fig:speedup}
    \begin{minipage}[l]{0.6\columnwidth}
      \centering
      \includegraphics[width=1\textwidth]{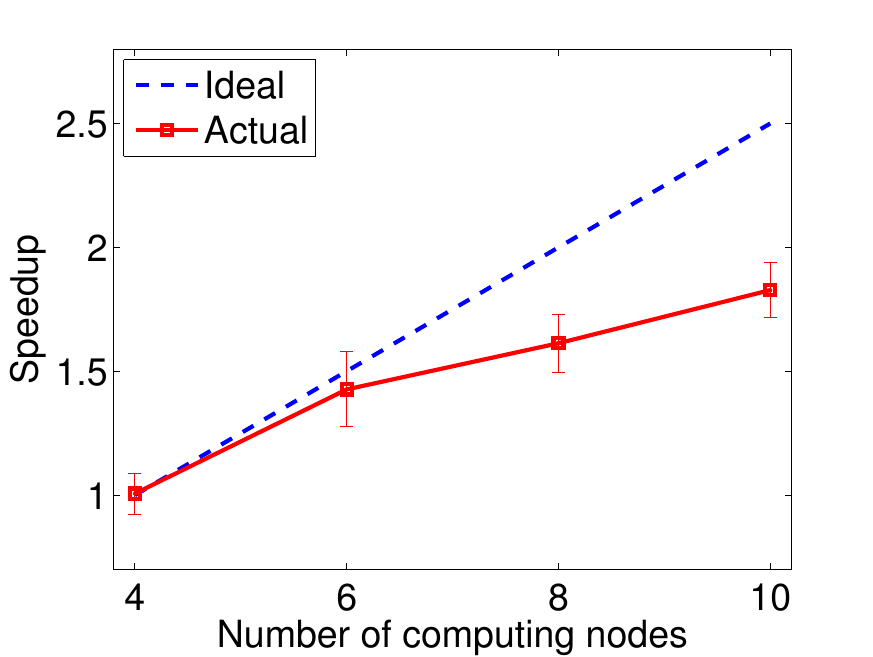}
    \end{minipage}
  }
\caption{Sensitivity analysis of hyperparameters and the speedup of the distributed learning framework.}\label{fig:hyperandspeedup}
\end{figure*}

\subsection{Scalability}\label{sec:speedup}

To investigate the scalability of our distributed learning framework introduced in Section~\ref{sec:spark}, we compute the speedup factor relative to the time cost with 4 nodes by varying the number of computing nodes from 4 to 10. The number of training data partitions is always configured to be the number of cores. Experiments are repeated 10 iterations, and the average speedup factors with standard deviations are reported in Figure~\ref{fig:speedup}. We can observe that the speedup appears to be close to linear and close to the ideal speedup factors. Therefore, our distributed implementation of the MVM model is very scalable for web-scale applications.

The gap between the real and ideal speedup may result from the increasing communication cost with the increasing number of computing nodes, since copies of instance and feature vertices would be distributed in more computing nodes which leads to larger state synchronization cost of each instance and feature vertice. We will consider using the parameter server to alleviate such problem.

\section{Conclusion and Future Work}\label{sec:end}

In this paper, we have proposed a multi-view machine (MVM) model and presented a stochastic gradient descent (SGD) learning method with a distributed implementation on Spark. In general, the model is particularly designed for data that is composed of features from multiple views, between which the full-order interactions are effectively explored. In contrast to other models that include only partial feature interactions or factorize different orders of interactions separately, MVMs jointly factorize the full-order feature interactions and thereby benefiting parameter estimation under sparsity and rendering the model with the capacity to avoid overfitting. Moreover, MVMs can be applied to a variety of supervised machine learning tasks, including classification and regression. Empirical studies on real-world web application datasets demonstrate the effectiveness of MVMs on modeling feature interactions in multi-view data, which outperform baseline models for multi-view prediction.

The MVM model can be further investigated in several directions for future work. For example, in addition to SGD, we are interested in implementation of other learning algorithms in a distributed environment to facilitate convergence efficiency, {\em e.g.}, alternating least square (ALS) and Markov Chain Monte Carlo (MCMC) for MVMs. It is also interesting to have our model applied to other multi-view prediction problems. Moreover, defining an evaluation metric for an effective view segmentation would be critical for the subsequent multi-view learning.

\section{Acknowledgements}
We would like to thank the anonymous reviewers for their comments. We also thank Dinggang Shen for his discussions, Chunhui Zhang from the Bing Ads team for providing the dataset and Bo Zhao for helping with the experiments.

\balance
\bibliographystyle{plain}
\bibliography{reference}

\end{document}